\documentclass[10pt, conference, compsocconf]{IEEEtran}
\usepackage{amsmath,amscd,amssymb,graphics}
\usepackage{mathrsfs} 

\newtheorem{theorem}{Theorem}[section]
\newtheorem{corollary}[theorem]{Corollary} 
\newtheorem{lemma}[theorem]{Lemma}

\newtheorem{observation}[theorem]{Observation}
\newtheorem{remark}[theorem]{Remark}
\newtheorem{example}[theorem]{Example}

\newcommand{\abs}[1]{\lvert#1\rvert}

\def\E {{\mathbb E}}
\def\N{{\mathbb N}}

\def\R{{\mathbb R}}
\def\e{\varepsilon}

\ifCLASSINFOpdf
\else
\fi

\begin{document}
\title{A note on sample complexity of learning binary output neural networks under\\ fixed input distributions}

\author{\IEEEauthorblockN{Vladimir Pestov}
\IEEEauthorblockA{Department of Mathematics and Statistics\\
University of Ottawa\\
Ottawa, Ontario, Canada\\
vpest283@uottawa.ca}
}

\maketitle

\begin{abstract}
We show that the learning sample complexity of a sigmoidal neural network constructed by Sontag (1992) required to achieve a given misclassification error under a fixed purely atomic distribution can grow arbitrarily fast: for any prescribed rate of growth there is an input distribution having this rate as the sample complexity, and the bound is asymptotically tight. The rate can be superexponential, a non-recursive function, etc. We further observe that Sontag's ANN is not Glivenko--Cantelli under any input distribution having a non-atomic part.
\end{abstract}


\IEEEpeerreviewmaketitle

\section{Introduction}
We begin with a quote of the first part of the open problem 12.6 from Vidyasagar's book \cite{vidyasagar} (this problem appears already in the original 1997 version).

{\em
``How can one reconcile the fact that in distribution-free learning, every learnable concept class is also ``polynomially'' learnable, whereas this might not be so in fixed-distribution learning?

In the case of distribution-free learning of concept classes (...) there are only two possibilities:

\noindent
1. $\mathscr C$ has infinite VC-dimension, in which case $\mathscr C$ is not PAC learnable at all.
\\
2. $\mathscr C$ has finite VC-dimension, in which case $\mathscr C$ is not only PAC learnable, but the sample complexity $m_0(\e,\delta)$ is $O(1/\e+\log(1/\delta))$. Let us call such a concept class ``polynomially learnable''.

In other words, there is no ``intermediate'' possibility of a concept class being learnable, but having a sample complexity that is superpolynomial in $1/\e$. 

In the case of fixed-distribution learning, the situation is not so clear. (...) Is there a concept class for which {\em every} algorithm would require a superpolynomial number of samples? The only known way of consructing such a concept class would be to (...) attempt to construct a concept class whose $\e$-covering number grows {\em faster} than any exponential in $1/\e$.
It would be interesting to know whether such a concept class exists.''}

In fact, the existence of a concept class whose sample complexity grows exponentially in $1/\e$ under a given fixed input distribution was already shown in 1991 by Benedek and Itai \cite{BI} (Theorem 3.5). Their example consisted of all finite subsets of a domain.
Later and independently, a rather more natural concept class with such properties (generated by a neural network) was constructed by Barbara Hammer in her Ph.D. thesis \cite{hammer99a} (Example 4.4.3 on page 77), cf. also \cite{hammer99b}.

Here we somewhat strengthen the above results and at the same time show that the phenomenon is quite common.
Suppose that a concept class $\mathscr C$ satisfies a slightly stronger property than having an infinite VC dimension, namely: $\mathscr C$ shatters every finite subset of an infinite set. Fix a sequence $\e_k$ of desired values of learning precision, converning to zero, and let $f$ be an increasing real function on $[0,+\infty)$. Then one can find a probability measure $\mu$ on the domain $\Omega$ of $\mathscr C$ with the property that $\mathscr C$ is PAC learnable under $\mu$, but the sample complexity of learning to precision $\e_k$, $k=1,2,3,\ldots$, is growing as $\Omega(f(\e_k^{-1}))$. The prescribed rate of growth can be ridiculouly high, for instance, a non-recursive function. The bound is essentially tight. For example, a well-known sigmoidal feed-forward neural network of infinite VC dimension constructed by Sontag \cite{sontag} has this property.

This naturally brings up a question of behaviour of Sontag's network $\mathcal N$ under non-atomic input distributions. It follows from Talagrand's theory of witness of irregularity \cite{talagrand87,talagrand96} that $\mathcal N$ is not Glivenko--Cantelli with regard to any measure having a non-atomic part. We do not know if a similar property holds for PAC learnability, although it is easy to see non-learnability of $\mathcal N$ for some common measures (the uniform distribution on the interval, the gaussian measure). While discussing a relationship between Glivenko--Cantelli property, PAC learnability, and precompactness, we give an answer to another (minor) question of Vidyasagar.

Note that we find it instructive to present the above observations in the reverse order.
In Conclusion, we suggest a few open problems and a conjecture supported by the results of this note which might shed light on Vidyasagar's problem.

\section{Glivenko--Cantelli classes and learnability}

\subsection{PAC learnability and total boundedness}

Benedek and Itai \cite{BI} had proved that a concept class $\mathscr C$ is PAC learnable under a single probability distribution $\mu$ if and only if $\mathscr C$ is totally bounded in the $L^1(\mu)$-distance. Here we remind their results.

\begin{theorem}[Theorem 4.8 in \cite{BI}; Theorem 6.3 in \cite{vidyasagar}]
\label{th:suf}
Suppose $\mathscr C$ is a concept class, $\e>0$, and that $B_1,\ldots,B_k$ is an $\e/2$-cover for $\mathscr C$. Then the minimal empirical risk algorithm is PAC to accuracy $\e$. In particular, the sample complexity of PAC learning $\mathscr C$ to accuracy $\e$ with confidence $1-\delta$ is
\[m\leq \frac{32}{\e}\log\frac{k}{\delta}.\]
\end{theorem}

Recall that a subset $A$ of a metric space $X$ is {\em $\e$-separated}, or {\em $\e$-discrete,} if, whenever $a,b\in A$ and $a\neq b$, one has $d(a,b)\geq\e>0$. The largest cardinality of an $\e$-discrete subset of $X$ is the {\em $\e$-packing number} of $X$. For example, the following lemma estimates from below the packing number of the Hamming cube. 

\begin{lemma}[\cite{vidyasagar}, Lemma 7.2 on p. 279]
\label{l:2e}
Let $0<\e\leq 1/4$. The Hamming cube $\{0,1\}^n$, equipped with the normalized Hamming distance 
\[d_h(x,y) = \frac 1n \left\vert \{i\colon x_i\neq y_i\}\right\vert,\]
admits a family of elements which are pairwise at a distance of at least $2\e$ from each other of cardinality at least $\exp[2(0.5-2\e)^2n]$.
\end{lemma}

The following is a source of lower bounds on the sample complexity. 

\begin{theorem}[Lemma 4.8 in \cite{BI}; Theorem 6.6 in  \cite{vidyasagar}]
\label{th:nec}
Suppose $\mathscr C$ is a given concept class, and let $\e>0$ be specified. Then any algorithm that is PAC to accuracy $\e$ requires at least $\lg M(2\e,{\mathscr C},L^1(\mu))$ samples, where $M(2\e,{\mathscr C},L^1(\mu))$ denotes the $2\e$-packing number of the concept class $\mathscr C$ with regard to the $L^1(\mu)$-distance. 
\end{theorem}

For the most comprehensive presentation of PAC learnability under a single distribution, see \cite{vidyasagar}, Ch. 6.

\subsection{Glivenko--Cantelli classes}

A function class $\mathscr F$ on a domain (a standard Borel space) $\Omega$ is {\em Glivenko--Cantelli} with regard to a probability distribution $\mu$ (\cite{dudley}, Ch. 3), or else has the property of {\em uniform convergence of empirical means} ({\em UCEM} property) \cite{vidyasagar}, if for each $\e>0$
\begin{equation}
\label{eq:glivenko}
\sup_{\mu\in {\mathcal P}}\mu^{\otimes n}\left\{\sup_{f\in{\mathscr F}}\left\vert \E_{\mu}(f)-  \E_{\mu_n}(f) \right\vert\geq \e\right\}\to 0\mbox{ as }n\to\infty.
\end{equation}
Here $\mu^{\otimes n}$ is the product measure on $\Omega^n$, and $\mu_n$ stands for the empirical (uniform) measure on $n$ points, sampled from the domain in an i.i.d. fashion. We assume $\mathscr F$ to assume values in an interval (i.e., to be uniformly bounded).
The notion applies to neural networks as well, if $\mathscr F$ denotes the family of output functions corresponding to all possible values of learning parameters. 

Every Glivenko--Cantelli class $\mathscr F$ is PAC learnable, which explains the important role of this notion. In fact, every consistent learning rule $\mathcal L$ will learn $\mathscr F$. 
We find it instructive to give a different proof, replying in passing to a remark of Vidyasagar \cite{vidyasagar}, p. 241. After proving that every Glivenko--Cantelli concept class $\mathscr C$ with regard to a fixed measure $\mu$ is precompact with regard to the $L^1(\mu)$-distance,
the author remarks that his proof is both indirect (Glivenko--Cantelli $\Rightarrow$ PAC learnable $\Rightarrow$ precompact), and does not extend to function classes, so it is not known to the author whether the result holds if $\mathscr C$ is replaced with a function class $\mathscr F$.

The answer is yes, as is (implicitely) stated in \cite{talagrand96} (p. 379, the beginning of the proof of Proposition 2.5), but a deduction is also rather roundabout (proving first the absence of a witness of irregularity).
In fact, the result is really very simple.

\begin{observation}
Every (uniformly bounded) Glivenko--Cantelli function class $\mathscr F$ with regard to a fixed probabillty measure $\mu$ is precompact in the $L^1(\mu)$-distance.
\end{observation}

\begin{proof}
If $\mathscr F$ is not precompact, then for some $\e_0>0$ it contains an infinite $\e_0$-discrete subfamily $\mathscr F^\prime$. For every finite sample $\sigma\in\Omega^n$ there is a further infinite subfamily $\mathscr F^{\prime\prime}\subseteq \mathscr F^\prime$ of functions whose restrictions to $\sigma$ are at a pairwise $L^1(\mu_n)$-distance $<\e_0/2$ from each other (the pigeonhole principle coupled with the fact that the restriction of $\mathscr F$ to $\sigma$ is $L^1(\mu_n)$-precompact). This means that $\mu$- and $\mu_n$-expectations of some function of the form $\abs{f_1-f_2}$, $f_i\in{\mathscr F}$, $i=1,2$, differ between themselves by at least $\e_0/2$, and for at least one of $i\in\{1,2\}$,   \[\left\vert\E_{\mu}(f_i)-\E_{\mu_n}(f_i)\right\vert\geq \e_0/4\] (an application of the triangle inequality in $\R$). Since the latter is true for {\em every} sample, no matter the size, $\mathscr F$ is not Glivenko--Cantelli.
\end{proof}

In fact, the same proof works in a slightly more general case when $\mathscr F$ is uniformly bounded by a single function (not necessarily integrable).

This gives an alternative deduction of the implication {\em Glivenko--Cantelli} $\Rightarrow$ {\em PAC learnability}. 
Admittedly, the result obtained is somewhat weaker, as this way we do not get {\em consistent} learnability.

\subsection{Talagrand's witness of irregularity}

Talagrand \cite{talagrand87,talagrand96} had characterized uniform Glivenko--Cantelli function classes with regard to a single distribution
in terms of shattering. We will remind his main result for concept classes only. 
Let $\Omega$ be a measurable space, let $\mathscr C$ be a concept class on $\Omega$, and let $\mu$ be a probability measure on $\Omega$. A measurable subset $A\subseteq\Omega$ is a {\em witness of irregularity} of $\mathscr C$, if $\mu(A)>0$ and for every $n$ the set of all $n$-tuples of elements of $A$ shattered by $\mathscr C$ has full measure in $A^n$. In other words, $\mu$-almost all $n$-tuples of elements of $A$ are shattered by $\mathscr C$.

\begin{theorem}[Talagrand \cite{talagrand87}, Th. 2]
\label{th:talagrand}
A concept class $\mathscr C$ is Glivenko--Cantelli with regard to the probability measure $\mu$ if and only if $\mathscr C$ admits no witness of irregulaity.
\end{theorem}

Let $\mu$ be a probability measure on $\Omega$. Recall that a set $A$ is an {\em atom} if for every measurable $B\subseteq A$ one has either $\mu(B)=0$ or $\mu(B)=\mu(A)$. The measure $\mu$ is {\em non-atomic} if it contains no atoms, and {\em purely atomic} if the measures of atoms add up to one. The restriction of $\mu$ to the union of atoms is the {\em atomic part} of $\mu$.

Since a witness of irregularity can contain no atoms, the following is an immediate corollary of Talagrand's 1987 result.

\begin{corollary}
\label{c:pa}
If a measure $\mu$ is purely atomic, then every concept class $\mathscr C$ is uniform Glivenko--Cantelli with regard to $\mu$, and in particular PAC learnable.
\end{corollary}

The corollary is easy to prove directly, without using subtle results of Talagrand, and the result was observed (independently) in 1991 and investigated in detail by Benedek and Itai (\cite{BI}, Theorem 3.2). Notice that the result does not assert {\em polynomial} PAC learnability of $\mathscr C$, and we will see shortly that the required sample complexity of $\mathscr C$ can grow arbitrarily fast.

\subsection{The neural network of Sontag}

\begin{figure}[!t]
\begin{center}
\scalebox{0.3}[0.3]{\includegraphics{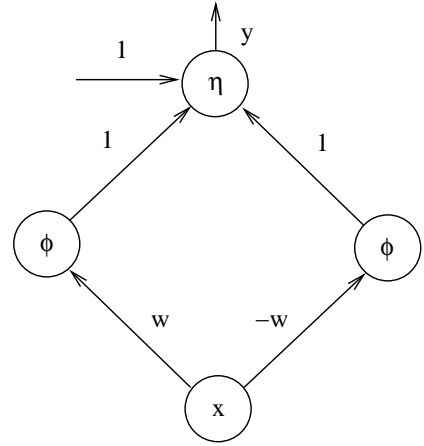}}
\scalebox{0.55}[0.55]{\includegraphics{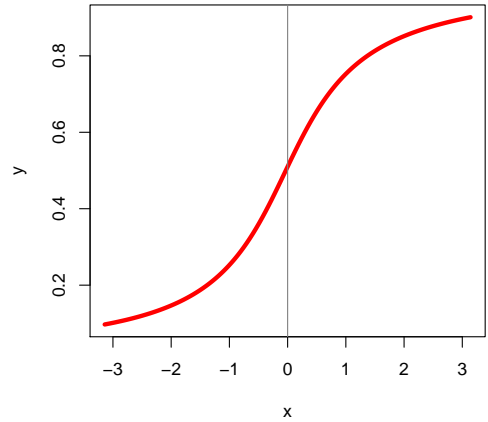}}
\end{center}
\caption{Sontag's ANN architecture (top) and the activation sigmoid $\phi$ with $\alpha=100$ (bottom).}
\label{fig:sontag}
\end{figure}


Figure \ref{fig:sontag} recalls a well-known example of a sigmoidal neural network $\mathcal N$ constructed by Sontag \cite{sontag}, pp. 34--36. (Cf. also \cite{vidyasagar}, page 389, where the top diagram in Figure \ref{fig:sontag} is borrowed from.) 
The activation sigmoid is of the form 
\[\phi(x) =\frac 1\pi\tan^{-1}x+\frac{\cos x}{\alpha(1+x^2)}+\frac 12,\]
where $\alpha\geq 2\pi$ is fixed, e.g. $\alpha=100$. 
and the output-layer perceptron has both input weights equal to one and a threshold of one. The input-output function of the network is given by
\[y = \eta[\rho(x)],\]
where
\[\rho(x) = \frac{2\cos wx}{\alpha(1+w^2x^2)}.\]
The input space of $\mathcal N$ is the space $\R$ of real numbers.

\begin{figure}[!t]
\begin{center}
\scalebox{0.55}[0.55]{\includegraphics{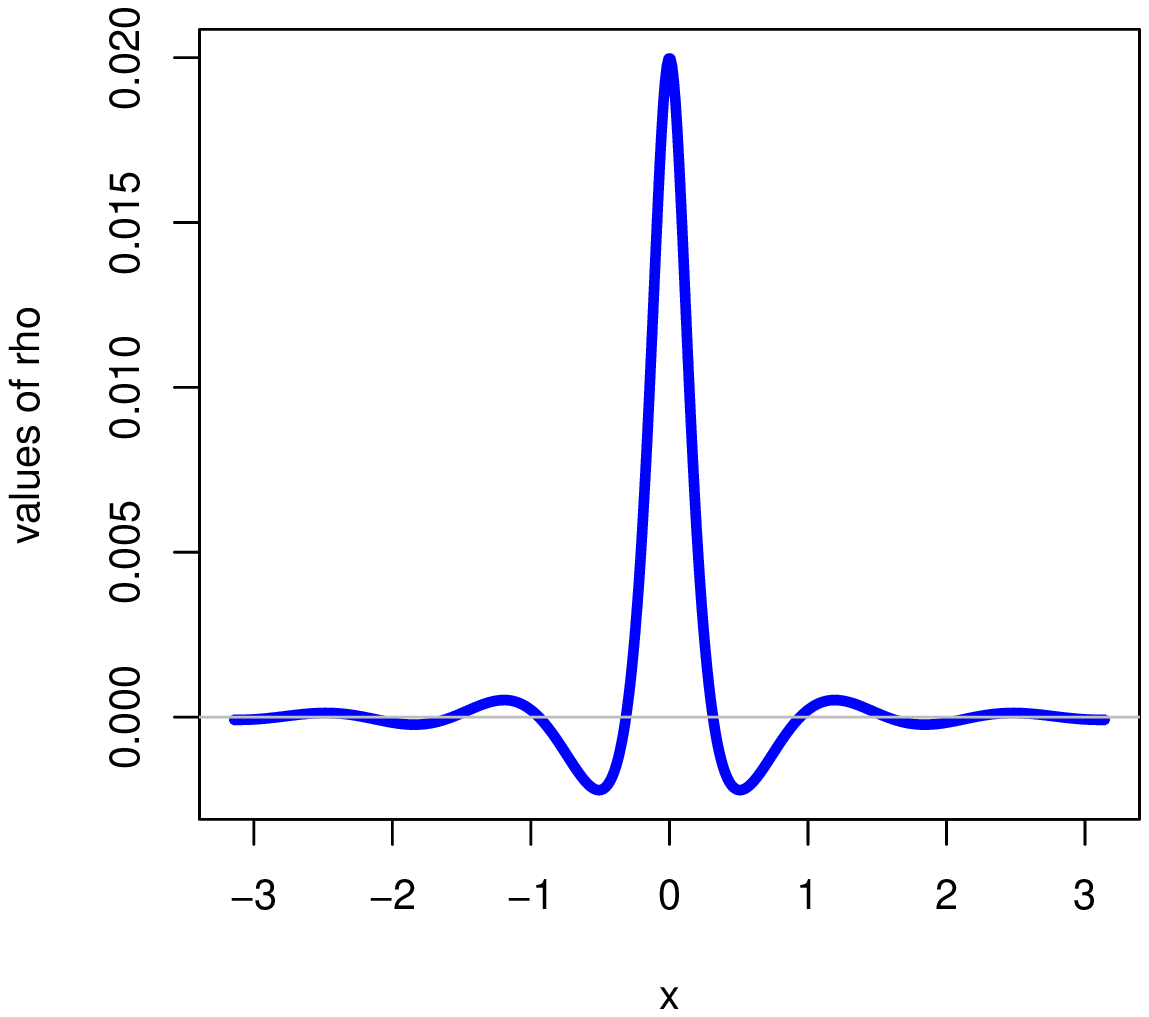}}
\scalebox{0.55}[0.55]{\includegraphics{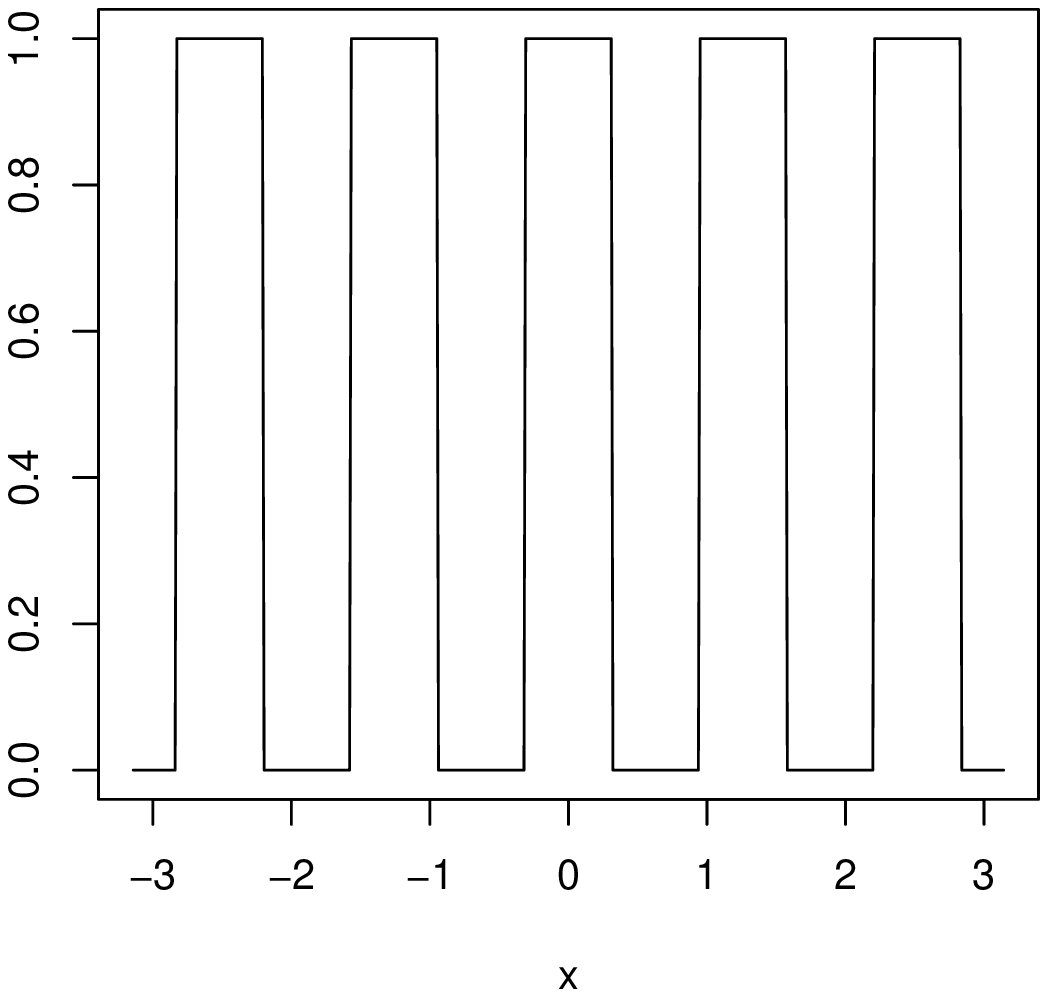}}
\end{center}
\caption{The function $\rho$ for $\alpha=100$ and $w=5$ (top) and the
corresponding output binary function (bottom).}
\label{fig:rho}
\end{figure}

%

Recall that a collection $x_1,x_2,\ldots,x_n$ of real numbers is {\em rationally independent} if no non-trivial linear combination of $1,x_1,x_2,\ldots,x_n$ with rational coefficients vanishes.

\begin{theorem}[\cite{sontag}, pp. 42-43]
\label{th:ratind}
The Sontag network $\mathcal N$ shatters every rationally independent $n$-tuple of real inputs $x_1,x_2,\ldots,x_n$. 
\end{theorem} 

In particular, the VC dimension of Sontag's network is infinite. Besides,
it is easy to find an infinite rationally independent set, and so every finite subset of such a set is shattered by $\mathcal N$. We will need this fact later.

Here is another extreme property of Sontag's network.

\begin{theorem}
\label{th:sontaggc}
The neural network of Sontag $\mathcal N$ is Glivenko--Cantelli under a probability distribution $\mu$ on the inputs if and only if $\mu$ is purely atomic.
\end{theorem}

\begin{IEEEproof}
Sufficiency $(\Leftarrow)$ follows from Corollary \ref{c:pa}. Let us prove necessity $(\Rightarrow)$. 
By splitting $\mu$ into a purely atomic part $\mu_{a}$ and a continuous part $\mu_{c}$, 
in view of Theorems \ref{th:talagrand} of Talagrand and \ref{th:ratind} of Sontag, it suffices to prove that for every non-atomic probability measure $\nu$ on $\R$ the set of rationally independent $n$-tuples has a full $\nu^{\otimes n}$ measure in $\R^n$: the support of $\mu_c$ will then be a witness of irregularity. In its turn, this reduces to a proof that for a fixed collection $(\lambda_1,\ldots,\lambda_{n+1})$ of rationals not all of which are zero, the affine hyperplane
\[H_{\lambda} = \{x\in\R^n\colon \langle x,\lambda\rangle=\lambda_{n+1}\},\]
where $\lambda =(\lambda_1,\ldots,\lambda_{n})$,
has $\nu^{\otimes n}$-measure null. This is a consequence of Eggleston's theorem \cite{eggleston}: If $A$ is a measurable,
Lebesgue-positive subset of the unit square, then there is a measurable positive set $B$ and a perfect set
$C$ such that $B\times C$ is included in $A$. ``Lebesgue measure on the unit square'' here is not a loss of generality, as every two non-atomic standard Borel probability measure spaces are isomorphic, and we obtain by induction that if $A\subseteq \R^n$ and $\nu^{\otimes n}(A)>0$, then $A$ contains a product of $n$ sets one of which is $\nu^{\otimes n}$-measure positive and all the rest are perfect (contain no isolated points). Clearly, no $(n-1)$-hyperplane in $\R^n$ can have this property.
\end{IEEEproof}

\begin{example}
Sontag's ANN is not PAC learnable under the uniform distribution on an interval.

Indeed, for the sequence of learning parameters $w_k=2^k$ the corresponding output binary functions are at a pairwise $L^1(\lambda)$-distance $1/2$ from each other, where $\lambda$ is a uniform distribution on some interval. 

A similar argument works for the gaussian distribution on the inputs.
\end{example}

However, we do not know if there exists a non-atomic measure under which Sontag's ANN is PAC learnable.

\subsection{Glivenko--Cantelli versus learnability}

Not every PAC learnable function, or even concept, class is Glivenko--Cantelli. Examples of such concept classes exist trivially, e.g. the concept class consisting of all finite and all cofinite subsets of the unit intervals is PAC learnable under every non-atomic distribution, yet clearly not uniform Glivenko--Cantelli, cf. \cite{BI}, p. 385, note (2), or \cite{vidyasagar}, p. 230, Example 6.4. A more interesting example, though based on the same idea, is Example 6.6 in \cite{vidyasagar}, p. 232. Here we present such an example of a countable concept class.

\begin{example}
For $n\in\N$, say that intervals $[i/n,(i+1)/n]$, $i=0,1,\ldots,n-1$, are of order $n$. 
Let $\mathscr C_n$ consist of all unions of less than $\sqrt n$ intervals of order $n$, and set ${\mathscr C}=\cup_{i=1}^\infty {\mathscr C}_n$.
If now $k\in\N$ is any and $x_1<x_2<\ldots<x_k$ are points of the unit interval, choose $n>k^2$ so that $1/n$ is smaller than any of the half-distances between neighbouring points $(x_{i+1}-x_i)/2$, $i=1,2,\ldots,n$. Clearly, elements of ${\mathscr C}_n$ shatter the sample $\{x_1,x_2,\ldots,x_k\}$, and so the entire interval is a witness of irregularity for the concept class $\mathscr C$. 
By Talagrand's result, the class $\mathscr C$ is not Glivenko--Cantelli. At the same time, for every $n$, ${\mathscr C}_n$ forms an $n^{-1/2}$-net for $\mathscr C$ with regard to the $L^1(\lambda)$-distance, and so $\mathscr C$ is PAC learnable under the Lebesgue measure $\lambda$ (the uniform measure on the interval). 

Observe that, in fact, $\mathscr C$ fails the Glivenko--Cantelli property with regard to {\em every} measure having a non-atomic part. As we have seen, there exist non-atomic measures under which $\mathscr C$ is PAC learnable. There are also measures under which $\mathscr C$ is not PAC learnable. for example the Haar measure $\nu$ on the Cantor set.

Recall the construction of the Cantor ``middle third'' set $\mathcal C$ (Figure \ref{fig:cantor}). This is the set left of the closed unit interval $[0,1]$ after first deleting the middle third $(1/3,2/3)$, then deleting the middle thirds of the two remaining intervals, $(1/9,2/9)$ and $(7/9,8/9)$, and continuing to delete the middle thirds {\em ad infimum}. The elements of the Cantor set are exactly those real numbers between $0$ and $1$ admitting a ternary expansion not containing $1$. Sometimes $\mathcal C$ is called {\em Cantor dust}. The complement to the Cantor set is a union of countably many open intervals, all the middle thirds left out.
The set $\mathcal C_n$ left after the first $n$ steps of removing the middle thirds is the union of $2^n$ closed intervals of equal length $3^{-n}$ each. The Haar measure of every such interval is set to be equal to $2^{-n}$, and this condition defines a non-atomic measure $\nu$ supported on $\mathcal C$ in a unique way. 

\begin{figure}[!t]
\begin{center}
\scalebox{0.275}[0.275]{\includegraphics{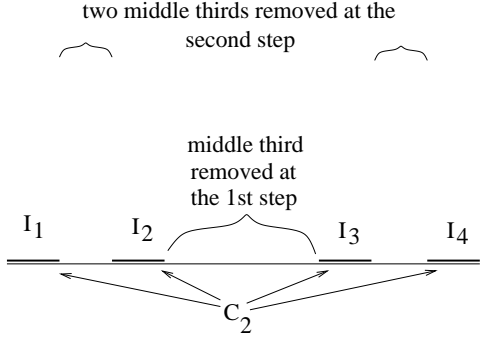}}
\end{center}
\caption{Construction of the Cantor set, after $n=2$ steps.}
\label{fig:cantor}
\end{figure}

It is easy to see now that the closed intervals $I_1,I_2,\ldots,I_{2^n}$ at the level $n$ are shattered with concept classes from ${\mathscr C}_N$ if $N$ is large enough ($\geq 2^{2n}$), in the following sense: for every set of indices $J\subseteq \{1,2,\ldots,2^n\}$ there is a $C\in {\mathscr C}_N$ which contains every interval $I_j$, $j\in J$, and is disjoint from every interval $I_k$, where $k\notin J$. Now one can modify the proof of Lemma \ref{l:2e} exactly as it was done in \cite{pestov}, proof of Theorem 3, in order to conclude that $\mathscr C$ is not totally bounded in the $L^1(\nu)$-distance. 
\end{example}

\section{All rates of sample complexity are possible}

\begin{theorem}
\label{th:main}
Let $\mathscr C$ be a concept class which shatters every finite subset of some infinite set. Let $(\e_k)$, $\e_k\downarrow 0$ be a sequence of positive reals converging to zero, and let $f\colon\R_+\to\R_+$ be a non-decreasing function growing at least linearly: $f(x) = \Omega(x)$. Then there is a probability measure $\mu = \mu((\e_k),f)$ on the input domain $\Omega$ with the property that for every $\delta>0$ and $k\in\N$ the class $\mathscr C$ is PAC learnable under the distribution $\mu$ to accuracy $\e_k$, and the rate of required sample complexity is at least 
\begin{equation}
\label{eq:sample}
n(\e_k,\delta)=\Omega\left(f\left(\frac{1}{\e_k}\right) \right).\end{equation}
Moreover, the above estimate is essentially tight in the sense that the sample complexity 
\begin{equation}
\label{eq:sample3}
n(\e_k,\delta)=O\left(f\left(\frac{1}{\e_k}\right)+
\log\left(\frac{1}{\delta}\right)\right).\end{equation}
suffices to learn $\mathscr C$ to accuracy $4\e_k$ with confidence $1-\delta$.
\end{theorem}

\begin{IEEEproof}
We can assume without loss in generality that $\e_1 = 1/5$. For every $k$, set $m_k = 5(\e_{k+1}-\e_k)$. Then $m_k$ form a sequence of non-negative reals which sums up to one. Denote, for simplicity, $f_k=f(\e_k^{-1})$.
Further, choose pairwise disjoint finite sets $F_k$ of cardinality $\abs{F_k}=f_{k}-f_{k-1}$ (where $f_0=0$) in a way that every union of finitely many of $F_k$'s is shattered by $\mathscr C$ (this is possible due to the assumption on the class $\mathscr C$). Let $\mu_k$ denote a uniform measure supported on $F_k$ of total mass $m_k$. Now set $\mu=\sum_{i=1}^\infty \mu_k$. Since $\sum_{i=1}^\infty m_k=1$, $\mu$ is a probability Borel measure. 

Let $k$ be arbitrary. Select any subset of $\mathscr C$ shattering $\cup_{i=1}^k F_i$ and containing 
\[\prod_{i=1}^k \abs{F_i}= 2^{f_k}\]
elements. This set forms a finite $\e_k$-net in $\mathscr C$ with regard to the $L^1(\mu)$-distance. Since $\e_k\downarrow 0$, we use Theorem \ref{th:suf} to conclude: the class $\mathscr C$ is PAC learnable under $\mu$, and the sample complexity of learning $\mathscr C$ to accuracy $\e_k$ and confidence $1-\delta$, $\delta>0$ is
\[m\geq \frac{8}{\e^2}\log\frac{2^{f_k}}{\delta} = \frac{8}{\e_k^2}\left(f_k +\log(\delta^{-1}) \right).\]

For every $k$, Lemma \ref{l:2e}, applied with $\e=0.2$, guarantees the existence of a subset $\Phi_k$ of $\mathscr C$ every two elements of which are at a 
$L^1(\mu_i)$-distance $\geq 0.42m_i$ from each other, and containing $\geq \exp[0.0128(f_k-f_{k-1})]$ elements. Let $N$ be so large that $\sum_{k=1}^Nm_k\geq (1.05)^{-1}$. Fix $k$. 
Since $\cup_{k=1}^N F_k$ is shattered by $\mathscr C$, one can find elements of $\mathscr C$ which correspond to elements of the product $\prod_{i=k}^N \Phi_i$, and every two of which are at a distance $\geq 0.42\sum_{k=1}^N m_k\e_k\geq 0.4\e_k$ from each other. According to Theorem \ref{th:nec}, this means that the computational complexity of learning $\mathscr C$ under $\mu$ to accuracy $\e_k$ with confidence $1-\delta$ is at least $0.0128 f_k$ samples.
\end{IEEEproof}

\begin{remark}
The measure $\mu$ constructed in the proof is purely atomic. However,
by replacing the domain $\Omega$ with $\Omega\times [0,1]$, every concept $C\in {\mathscr C}$ with $C\times [0,1]$, and $\mu$ with the product $\mu\otimes\lambda$, where $\lambda$ is the uniform (Lebesgue) measure on the interval, one can ``translate'' every example as above into an example of learning under a non-atomic probability distribution.
\end{remark}

\begin{corollary}
\label{c:mu}
Let $\nu$ be a probability distribution on a domain $\Omega$ having infinite support. Then there exist concept classes $\mathscr C$ which are PAC learnable under $\nu$ and whose required sample complexity is arbitrarily high.
\end{corollary}

\begin{IEEEproof}
The measure space $(\Omega,\nu)$ admits a measure-preserving map $\phi$ to the measure space constructed in the proof of Theorem \ref{th:main} in such a way that $\nu\phi^{-1}=\mu$ (here one uses the fact that $\mu$ is purely atomic). Now the concept class ${\mathscr C}\phi^{-1}$, consisting of all sets $\phi^{-1}(C)$, has the same learning properties under the distribution $\nu$ as the class $\mathscr C$ has under $\mu$. 
\end{IEEEproof}

\begin{corollary} 
\label{c:sontag}
Let $\e_k\downarrow 0$ be a sequence of positive values converging to zero, and let $f_k$ be a real function on $[0,+\infty)$ growing at least linearly. Then there is a probability distribution $\mu$ on the real numbers under which Sontag's network $\mathcal N$ is PAC learnable to accuracy $\e_k$ with confidence $1-\delta$, requiring the sample of size $\Omega(f(\e_k^{-1}))$. This estimate is essentially tight, because the sample size
\begin{equation}
\label{eq:sample4}
n(\e_k,\delta)=O\left(f\left(\frac{1}{\e_k}\right)+
\log\left(\frac{1}{\delta}\right)\right).\end{equation}
already suffices to train $\mathcal N$ to accuracy $4\e_k$ with confidence $1-\delta$.
\end{corollary}

\begin{remark}
\label{r:easy}
It is easy to construct concept classes which are PAC learnable under {\em every} input distribution, and yet exhibit all possible rates of learning sample complexity. These are the classes $\mathscr C$ which, speaking informally, cannot tell a difference between a given probability distribution $\mu$ and some purely atomic measure $\nu$. More precisely, if the sigma-algebra of sets generated by $\mathscr C$ is purely atomic and $\mathscr C$ shatters every finite subset of an infinite set, then $\mathscr C$ will have the above property. 

An example is a class $\mathscr C$ that consists of all finite unions of middle thirds of the Cantor set $\mathcal C$. The atoms of the sigma-algebra of sets generated by this class are precisely the middle thirds, and so $\mathscr C$ has the desired property.
\end{remark}

\section{Conclusion}
Stimulated by a  question embedded into the Problem 12.6 of Vidyasagar \cite{vidyasagar},
we have shown that all rates of sample compleixity growth are possible for distribution-dependent learning, in particular all are realized by binary output feed-forward sigmoidal neural network of Sontag. Now Vidyasagar continues thus:

{\em 
``I would like to have an ``intrinsic'' explanation as to why in distribution-free learning, every learnable concept class is also forced to be polynomially learnable. Next, how far can one ``push'' this line of argument? Suppose $\mathcal P$ is a family of probabilities that contains a ball in the total variation metric $\rho$. From Theorem 8.8 it follows that every concept class that is learnable with respect to $\mathcal P$ must also be polynomially learnable (because $\mathscr C$ must have finite VC-dimension). Is it possible to identify other such classes of probabilities?''
}

We suggest the following conjecture, which, in our view, is the right framework in which to address Vidyasagar's question.

{\em Conjecture} ({\em ``the sample complexity alternative''}). Let $\mathcal P$ be a family of probability distributions on the domain $\Omega$. Then either every class learnable under $\mathcal P$ is learnable with sample complexity $O(\e^{-1})$, or else there exist PAC learnable classes under $\mathcal P$ whose required sample complexity grows arbitrarily fast. 

The classical VC theory tells that the conjecture is true if $\mathcal P$ is the family of all probability measures: namely, the first alternative holds always.
In view of Corollary \ref{c:mu}, the conjecture is also true in the other extreme case, where $\mathcal P=\{\mu\}$ contains a single distribution: unless $\mu$ is finitely-supported, we have the second alternative. 

{\em Problem 1.} Does the above alternative hold for every family $\mathcal P$ of probability distributions on the inputs?

{\em Problem 2.} Does there exist a non-atomic probability measure on $\R$ under which the Sontag ANN is PAC learnable?
\newpage

{\em Problem 3.} Give a criterion for a concept class to be PAC learnable under a fixed probability distribution in terms of shattering. 

Some sufficient conditions can be found in \cite{BI,AST}, but none of them is also necessary. The ``right'' condition will be strictly intermediate between the witness of irregularity \cite{talagrand87,talagrand96} and the VC dimension modulo countable sets \cite{pestov}.  

\section*{Acknowledgments}
The author is grateful to the anonymous referees, in particular for pointing out the references \cite{hammer99a,hammer99b}, and to Ilijas Farah for pointing out the reference \cite{eggleston}.

\end{document}